\documentclass[journal, 10pt]{resources/IEEEtran}
\IEEEoverridecommandlockouts                 % This command is only needed if 
\usepackage[utf8]{inputenc}

\usepackage{amsthm, amsmath, amsfonts, amssymb}
\usepackage{resources/custom_commands}
\usepackage{resources/coloredboxes}

\usepackage{enumerate}
\usepackage{graphicx}
\usepackage{mathtools}
\usepackage{thmtools}
\usepackage{thm-restate}
\usepackage{float}
\usepackage{mdframed} % For creating framed boxes
\usepackage{enumerate}
\usepackage{marginnote}
\usepackage[hidelinks]{hyperref}
\usepackage{cleveref}
\usepackage{subfigure}
\usepackage{tikz}
\usetikzlibrary{shapes.geometric, arrows, positioning, arrows.meta}

\tikzset{
    box/.style={
        rectangle,
        rounded corners,
        draw=black, very thick,
        text width=6.5em,
        minimum height=2em,
        text centered},
    arrow/.style={-{Stealth[]}, very thick}
}

\usepackage[shortlabels]{enumitem}

\renewcommand{\bar}[1]{\overline{#1}}
\renewcommand{\tilde}[1]{\widetilde{#1}}
\renewcommand{\hat}[1]{\widehat{#1}}

\newcommand{\Prob}{\mathbb P}
\newcommand{\Probs}{\mathbb{P}_{\theta^*}}
\newcommand{\E}{\mathbb E}
\newcommand{\Es}{\mathbb{E}_{\theta^*}}

\newcounter{relctr} %% <- counter for relations
\everydisplay\expandafter{\the\everydisplay\setcounter{relctr}{0}} %% <- reset every eq
\newcommand\labelrel[2]{%
  \begingroup
    \refstepcounter{relctr}%
    \stackrel{\textnormal{(\alph{relctr})}}{\mathstrut{#1}}%
    \originallabel{#2}%
  \endgroup
}
\AtBeginDocument{\let\originallabel\label}

\newif\iflong

\longtrue
\newif\ifsixteen

\sixteentrue
\begin{document}

\title{\LARGE \bf
%Convergence of Opinion Dynamics under Social Pressure for General Networks
Belief Samples Are All You Need For Social Learning}
\author{Mahyar JafariNodeh, Amir Ajorlou, and Ali Jadbabaie
\thanks{This work was supported by ARO MURI W911NF-19-1-0217.The authors are with the Institute for Data, Systems, and Society (IDSS), Massachusetts Institute of Technology (MIT),
Cambridge, MA 02139, USA.  {\tt\small \{mahyarjn, ajorlou,   jadbabaie\}@mit.edu}.}}
\date{\today}

\maketitle

% next two lines for page number in IEEEtran
\thispagestyle{plain}
\pagestyle{plain}

\begin{abstract}
In this paper, we consider the problem of social learning, where a group of agents embedded in a social network are interested in learning an underlying state of the world. Agents have 
incomplete, noisy,  and heterogeneous  sources of information, providing them with recurring private observations of the underlying state of the world. Agents can share their learning experience with their peers by taking actions observable to them, with values from a finite feasible set of states. Actions can be interpreted as samples from the beliefs which agents may form and update on what the true state of the world is. Sharing samples, in place of full beliefs, is motivated by the limited communication, cognitive, and information-processing resources available to agents especially in large populations. Previous work (\cite{Rabih_CDC_2020}) poses the question as to whether learning with probability one is still achievable if agents are only allowed to communicate samples from their beliefs. We provide a definite positive answer to this question, assuming a strongly connected network and a ``collective distinguishability'' assumption, which are both required for learning even in full-belief-sharing settings.
In our proposed belief update mechanism, each agent's belief is a normalized weighted geometric interpolation between a fully Bayesian private belief --- aggregating information from the private source --- and an ensemble of empirical distributions of the samples shared by her neighbors over time. By carefully constructing asymptotic almost-sure lower/upper bounds on the frequency of shared samples matching the true state/or not, we rigorously prove the convergence of all the beliefs to the true state, with probability one. 
\end{abstract}

%\tableofcontents
\section{Introduction and related work}
In recent years, there has been a surge in research exploring  mechanisms of belief formation and evolution in large populations, where individual agents have information of varying quality and precision, information exchange is limited and localized, and the sources, reliability, and trustworthiness of information is unclear. 
% At the heart of this exploration lies a series of innovative learning models, each providing a unique lens through which to view the confluence of personal experience and communal knowledge.
The body of literature on social learning, particularly within the realm of non-Bayesian models, reveals a nuanced landscape where individual cognitive capabilities, network structures, and the flow of information converge to shape collective outcomes. 

The DeGroot model presented in \cite{612bb50a-4bdd-3a32-b6eb-7837600cc9c4} is a  simple model of consensus formation, where individuals update their beliefs by taking weighted averages of their neighbors' beliefs. This model provided a mathematical framework for analyzing the convergence of beliefs in a network setting.  Authors in \cite{10.1257/mic.2.1.112} have examined how the structure of social networks influences the accuracy of collective belief formation, highlighting the importance of network centrality and the distribution of initial opinions. Conditions under which communities can learn the true state of the world---despite the presence of biased agents--- have been investigated in \cite{10.1093/restud/rdr004}, contributing to our understanding of the robustness of social learning processes to misinformation and bias. \cite{a6bfde8b-88e1-3655-95a5-006c96403970} explored the implications of limited information processing capabilities on social learning outcomes, demonstrating how cognitive constraints can lead to the persistence of incorrect beliefs within networks. The work in
\cite{doi:10.1126/science.aac6076} focused on computational rationality, providing valuable insights into how individuals make decisions under uncertainty by approximating Bayesian inference, relevant for understanding the cognitive underpinnings of social learning. 

More recently, authors in \cite{https://doi.org/10.3982/ECTA14613} offered a comprehensive analysis of non-Bayesian social learning, identifying the fundamental forces that drive learning, non-learning, and mislearning in social networks. 
Another closely related work is \cite{jadbabaie2012}, where agents make recurring private noisy observations of an underlying state of the world and repeatedly engage in communicating their beliefs on the state with their peers. 
Agents use Bayes rule to update their beliefs upon making new observations. 
Subsequently and after receiving her peers' beliefs in each round, each agent then updates her belief to a convex combination of her own belief and those of her peers. It is then shown that under the so called ``collective distinguishability assumption'' and provided a strongly connected communication network, all agents learn the true state with probability one. 

A key behavioral assumption in many approaches to non-Bayesian social learning (including \cite{https://doi.org/10.3982/ECTA14613, jadbabaie2012})
is that agents are capable of repeatedly communicating their full
belief distributions with their peers. 
As pointed out in \cite{Rabih_CDC_2020}, 
decision-makers in large populations are likely not
to satisfy such a cognitive demand, given the limited/costly
communication and information processing resources.
Motivated by such limitations, authors in \cite{Rabih_CDC_2020} pose the question as to whether almost sure learning is achievable if agents are only allowed to communicate samples from their beliefs. They analyze the learning process under a sample-based variation of the model in \cite{jadbabaie2012}, and show that collective distinguishability is not sufficient for learning anymore.\footnote{The potential for mislearning when relaying actions instead of information is also underscored in \cite{10.1257/jep.12.3.151, 571c752f-4235-30d0-a01d-58e2ef1b884d}.
}

In this paper, we contribute to this line of work by 
proposing a framework where agents only communicate samples from their beliefs, and yet learning is achievable with probability one. 
Each agent's belief in our model is a geometric interpolation between a fully Bayesian private belief --- aggregating information from a private source --- and an ensemble of empirical distributions of the actions shared by her neighbors (normalized to add up to 1). By carefully constructing asymptotic almost sure lower/upper bounds on the frequency of the shared actions communicating the true/wrong state, we prove the convergence of all the beliefs to the true state with probability one.

\section{Mathematical model}\label{sec:model}

We consider a set of $n$ agents 
denoted by $[n]=\{1,\dots,n\}$,
who aim to learn 
an underlying state of the world 
$\theta$. This state is a random variable on a probability space $(\Omega,\mathcal{F},\mathbb{P})$ and takes values in a 
finite set $\Theta$, and take its size to be $m$ (i.e. $|\Theta| = m$). 

We adopt the same information structure as in \cite{Rabih_CDC_2020}: 
At each time period $t = 1,2, \ldots$ and  conditional on the state
$\theta$, each agent $i$ observes a private signal 
$\omega_{it} \in S_i$ generated by the likelihood function $l_i(\cdot|\theta) \in \Delta_{S_i}$.
Here, the finite set $S_i$ denotes agent $i$'s signal space and $\Delta_{S_i}$ the set of probability measures on $S_i$.
 We denote the profile of each agent's signals by $\omega_{i}^{t} := (\omega_{i1}, \ldots, \omega_{it})$. 
We assume that the observation profiles $\{\omega_{it}\}_{i=1}^{n}$ are independent over time, and that $l_{i}(\omega_i|\theta)>0$ for all $i \in [n]$ and $(\omega_i,\theta)\in S_i\times\Theta$.

It is to be noted that agents, in general, may not be able to identify the true state solely relying on their private observations. This is the case when two states are are observationally equivalent to an agent: Two states $\theta \neq \theta'$ are observationally equivalent to agent $i$ if $l_i(\cdot|\theta)=l_i(\cdot|\theta')$. As a remedy,
agents engage in repeated communication with each other on a social network where they can make state-related observations from their neighbors. The network is a
weighted directed graph paramterized with $(\mathcal{V}, \mathcal{E})$ with adjacency matrix $A=\{a_{ij}\}_{i,j \in [n]^2}$, where the weights $a_{ij}$ are non-negative and $\sum_{j=1}^n a_{ij}=1$. A positive weight $a_{ij}>0$ implies that agent $j$ is a neighbor of agent $i$, and in particular, agent $i$ can observe the action of agent $j$. We show the set of neighbors of agent $i$ with $\mathcal{N}_{i}$. We assume agents have positive self-confidences, that is, the diagonal entries of $A$ are all positive.

In our framework, 
each agent constructs an empirical distribution of their neighbors' actions. We denote agent $j$'s action at time $t$ by $c_{jt}\in \Theta$ and the profile of her actions by $c_{j}^{t} := (c_{j1}, \ldots, c_{jt})$; The indicator function ${\bf 1}_{c_{jt}}(\theta)$ is then equal to 1 if agent $j$ has declared $\theta$ at time $t$ as her opinion. Agents can use their actions as a means to broadcast 
their opinion on which state they find more likely to be the true state of the world to their neighbors. We elaborate on our proposed strategy for taking actions later in this section.
% To formalize how the process of opinion sharing is done, we define an empirical distribution over users ($\forall j\in [n]$), constructed by their neighbors based on the agents profile of opinion sharing which we donote as $\hat{\bmu}_{jt} \in \mathbf{R}^{k}$. 
Neighbors of agent $j$ construct an empirical distribution $\hat{\bmu}_{jt} \in \Delta_{\Theta}$ of her actions by taking counts of the times she declares $\theta$ as her opinion/action for each $\theta\in\Theta$. 

% We initialize all empirical distributions as well as private belief to be uniform, i.e., $\forall j\in [n]: \bmu_{j0}^{P}=\hat{\bmu}_{j0} = (\frac{1}{k}, \dots, \frac{1}{k})$.
For each $\theta\in\Theta$, let $n_{jt}(\theta) := 1+\sum_{\tau=1}^{t} {\bf 1}_{c_{j\tau}}(\theta)$ count how many times agent $j$ takes action $\theta$ up to time $t$. We initialize all counters by $1$. We then normalize the counts to construct what we refer to (with a bit misuse of notation) as the empirical distribution of declared actions for agent $j$:
\begin{align}
\label{eq:empiricals}
\hat{\bmu}_{jt}(\theta):=\frac{n_{jt}(\theta)}{\sum{n_{jt}(\theta')}_{\theta'\in\Theta}}=\frac{n_{jt}(\theta)}{t+m}.
\end{align}

Each agent also holds a private belief $\bmu^{P}_{it}\in \Delta_{\Theta}$ aggregating information from its private source following Bayes update rule:
\begin{equation} \label{eq:personal_update}
    \mu_{it}^{P}(\theta|\omega_{i}^{t}) = \frac{l_{i}(\omega_{it}\mid\theta) \cdot \mu_{it-1}^{P}(\theta|\omega_{i}^{t-1})}{m_{it}(\omega_{it})},
\end{equation}
\begin{equation*}
    m_{it}(\omega_{it})= \sum_{\theta \in \Theta} l_{i}(\omega_{it}|\theta) \mu_{it-1}^{P}(\theta|\omega_{i}^{t-1}).  
\end{equation*}
We initialize the private beliefs to be uniform, i.e., $\bmu_{i0}^{P} = (\frac{1}{m}, \dots, \frac{1}{m})$.
Each agent $i$ then incorporates the empirical distribution of declared opinions of her neighbors $\hat{\bmu}_{jt}$ for all $j\in \mathcal{N}_{i}$ into their private belief $\bmu^{P}_{it}$ to form her belief $\bmu_{it}\in\Delta_{\Theta}$ on what true state of the word is. 
They do so by taking the weighted geometric mean of their private beliefs and the empirical distribution of their neighbors' actions, and normalizing it to add up to 1:
\begin{equation} \label{eq:update_rule}
    \mu_{it}(\theta)  \propto \mu_{it}^{P}(\theta)^{a_{ii}}\times \prod_{j \in \mathcal{N}_{i}} \hat{\mu}_{jt}(\theta)^{a_{ij}}.
\end{equation}
Notice that the weights $a_{ii}$ and $a_{ij}$'s  capture the trust of agent $i$
in her private source of information and her neighbors' declared opinions, respectively. 
% For now on, we will use the notations $\mu_{it}(\theta|\omega_{i}^{t}, c_{\mathcal{N}_{i}}^{t})$, and $\mu_{it}(\theta)$ interchangeably.
Each agent $i$ then takes action $c_{it}$ by drawing a sample from her belief $\mu_{it}$ (i.e. $c_{it}\sim \mu_{it}$), which is subsequently observed by those who are neighboring her.

\section{Model discussion and preliminaries} \label{section:b_d}

The key contribution of this work is to show that as long as the agents can collectively distinguish
the states and the graph is strongly connected, learning
occurs with probability one under our proposed framework. Collective distinguishability
means that for every two different states $\theta$ and $\theta'$, there exists
an agent $i$ such that $l_i(.|\theta)\neq l_i(.|\theta')$. We formally define learning below.
\begin{definition}[\cite{Rabih_CDC_2020}] \label{def:learning}
    Agent $i\in[n]$ learns the true state $\theta^*$ along the sample path
    $w \in \Omega$, if  $\lim\limits_{t \to \infty}\mu_{it}(\theta^*)=1$ at $w$. 
\end{definition}

It proves insightful to elaborate on connections/distinctions of our work with \cite{jadbabaie2012, Rabih_CDC_2020} which study models similar to ours.   
Instead of sharing samples from beliefs, authors in \cite{jadbabaie2012} assume that agents are capable of sharing their full beliefs with their neighbors in each round. Their belief update rule is of the form

\begin{equation} \label{update_rule_pooya}
    \mu_{it+1}(\theta) = a_{ii} \frac{l_i(\omega_{it+1}|\theta)}{m_{it}(\omega_{it+1})} \mu_{it}(\theta) + \sum_{j \in \mathcal{N}_i} a_{ij} {\mu_{jt}}(\theta),
\end{equation}
where 
\begin{equation*}
    m_{it}(\omega)= \sum_{\theta \in \Theta} l_i(\omega|\theta) \mu_{it} (\theta).  
\end{equation*}
% \begin{equation} \label{update_rule_full}
%     \mu_{it+1} = a_{ii} \times 
%     \text{Bayes} (\mu_{it},\omega_{it}) 
%      + \sum_{j \in \mathcal{N}_i} a_{ij} \mu_{jt}.
% \end{equation}
They show that under the Collective distinguishability assumption and strongly connected graph, learning occurs with probability one. 

Motivated by the limited communication and cognitive resources available to agents especially in large populations, authors in \cite{Rabih_CDC_2020} pose the question as to whether learning with probability one is still achievable if agents are only allowed to communicate samples from their beliefs. They then analyze the learning process under a sample-based variation of \eqref{update_rule_pooya}:
\begin{equation*} \label{update_rule_rabih}
    \mu_{it+1}(\theta) = a_{ii} \frac{l_i(\omega_{it+1}|\theta)}{m_{it}(\omega_{it+1})} \mu_{it}(\theta) + \sum_{j \in \mathcal{N}_i} a_{ij} 1_{c_{jt}}(\theta).
\end{equation*}
As their main result, they prove that collective
distinguishability is not sufficient for learning in this case.
Our work complements this chain by proposing a framework where agents only communicate samples from their beliefs, and yet learning occurs with probability one. Each agent's belief is a geometric interpolation between a fully Bayesian private belief--aggregating information from a private source-- and an ensemble of empirical distributions of her neighbors' actions, as governed by \eqref{eq:empiricals}-\eqref{eq:update_rule}.\footnote{This is subsequently normalized to add up to 1.}  

\section{Main Results}\label{sec:results}
In this section, we rigorously analyze the belief dynamics governed by \eqref{eq:empiricals}-\eqref{eq:update_rule} to establish that learning occurs with probability one, under our proposed framework.
% investigate the asymptotic behavior of the update rule in \eqref{eq:update_rule} and its implications for learning the true state $\theta^{*} \in \Theta$. 

\begin{definition}
    Denoting the true state of the world by $\theta^*$, we say that a state $\theta \in \Theta$ is $\theta^*$-identifiable for agent $i$ if: 
    \[
        \theta \neq \theta^* \rightarrow l_i(.|\theta) \neq l_i(.|\theta^*).
    \]
\end{definition}
We also denote $\Prob_{\theta^{*}}(\cdot):=\Prob(\cdot|\theta^*)$, and $\E_{\theta^*}[\cdot]:=\E(\cdot|\theta^*)$. 

% \subsection{Sufficiency of Private Observations for Identifiable States}
\subsection{Exponentially Fast Decay of the Belief over Identifiable States}
We start by showing the exponential decay of private beliefs on the states identifiable from the true state.
% We first show that each agent can learn the true state (or come up with a uniform distribution on the true state and the set of nonidentifiable states) exponentially fast in time by using a Bayesian approach in the following Lemma:
We first introduce the notion of R\'enyi divergence.
\begin{definition}(\textbf{$\alpha$-R\'enyi divergence})
The $\alpha$-R\'enyi divergence between two discrete distributions $P$ and $Q$ is defined as,
\begin{equation*}
    D_{\alpha}(P\|Q):=\frac{1}{\alpha-1}\log\Big(\sum_{i=1}^k p_i^\alpha q_i^{1-\alpha}\Big), \quad \alpha\geq 0.
\end{equation*}
\end{definition}

% \begin{remark}
    Note that $D_{1}(P\|Q)$, that is the R\'enyi divergence for $\alpha=1$,  recovers KL-divergence.
% \end{remark}

\begin{lemma} \label{lemma:exponential_fast_convergence}
Let $\theta$ be a $\theta^*$-identifiable state for
agent $i \in [n]$. Then, for any $\beta_{i\theta}$ with $ 0<\beta_{i\theta}<D_{1}(l_{i}(.|\theta^*)\|l_{i}(.|\theta))$, there exists $\gamma_{i\theta} := \gamma(\beta_{i\theta}) > 0$ such that for each $t \in \mathbb{N}$ we have:
\begin{equation*}
    \Probs\left(\frac{\mu_{it}^{P}(\theta|\omega_{i}^{t})}{\mu_{it}^{P}(\theta^*|\omega_{i}^{t})} > e^{-\beta_{i\theta} t}\right) \leq e^{-\gamma_{i\theta}t}.
\end{equation*}
\end{lemma}
\medskip
\begin{proof}
    Using Equation~\eqref{eq:update_rule} one can write
    \begin{equation*}
        \frac{\mu_{it}^{P}(\theta|\omega_{i}^{t})}{\mu_{it}^{P}(\theta^*|\omega_{i}^{t})} = \left(\prod_{\tau=1}^{t}\frac{l_{i}(\omega_{i\tau}|\theta)}{l_{i}(\omega_{i\tau}|\theta^*)}\right)\cdot\frac{\mu_{i0}^{P}(\theta)}{\mu_{i0}^{P}(\theta^*)}.
    \end{equation*}
    Now by noting the independence of $\{\omega_{i\tau}\}_{\tau=1}^{t}$ due to i.i.d. samples, by Markov inequality, we have
    \begin{align*}
            &\Probs\left(\frac{\mu_{it}^{P}(\theta|\omega_{i}^{t})}{\mu_{it}^{P}(\theta^*|\omega_{i}^{t})} > e^{-\beta_{i\theta} t}\right) \leq \frac{ \Es\left[(\frac{l_{i}(\omega_{i}|\theta)}{l_{i}(\omega_{i}|\theta^*)})^{1-\alpha}\right]^{t}}{e^{-(1- \alpha)\beta_{i\theta} t}} \\
            &= \exp\left(-t(1-\alpha)(D_{\alpha}(l_{i}(.|\theta^*)\|l_{i}(.|\theta))-\beta_{i\theta})\right),
    \end{align*}
    where $\alpha \in (0, 1)$, and $D_{\alpha}(P\|Q)$ is $\alpha$-Renyi Divergence between distributions $P, Q$, and 
    $\gamma(\beta)= \max_{\alpha^*}(1-\alpha)(D_{\alpha}(l(.|\theta^*)\|l(.|\theta))-\beta)$ where $\alpha^{*}$ is the set of $\alpha \in (0, 1)$ for which $D_{\alpha}(l(.|\theta^*)\|l(.|\theta)) > \beta_{i\theta}$.
\end{proof}
% \begin{remark}
%     Note that one may argue choice of $\alpha$ could be extended to values in $(1, \infty)$, to further improve the derived probabilistic bounds, but note that for $\alpha \in (1, \infty)$ the direction of inequality would change which is not desired!
% \end{remark}
Lemma \ref{lemma:exponential_fast_convergence} suggests that as long as a state $\theta$ is identifiable for an agent, her belief on $\theta$ decays exponentially fast in time, no matter how misinforming her neighbors are, which is formally stated below.
\begin{lemma}\label{cor:1}
% For each agent $i \in [n]$ with $a_{ii}>0$, and the set of her $\theta^*$-identifiable states $\theta$, for sufficiently large $t$, we have
Let $\theta$ be a
$\theta^*$-identifiable state for agent $i \in [n]$. Then, for any
$0<\beta_{i\theta}<D_{1}(l_{i}(.|\theta^*)\|l_{i}(.|\theta))$ there exists $\gamma_{i\theta}> 0$ such that for sufficiently large $t$, we have
\begin{equation} \label{eq:cor1}
    \Probs(\mu_{it}(\theta) \geq e^{-a_{ii}{\beta}_{i\theta} t}) \leq e^{-\gamma_{i\theta}t},
\end{equation}
% where $0<\tilde{\beta}_{i\theta} < a_{ii}\beta_{i\theta}$ is chosen sufficiently small, with the same definition for $\beta_{i\theta}, \gamma_{i\theta}$ as Lemma \ref{lemma:exponential_fast_convergence}.
\end{lemma}
\begin{proof}
    Using Equation~\eqref{eq:update_rule} we have
    \begin{align}
        \frac{\mu_{it}(\theta)}{\mu_{it}(\theta^*)} &= \left(\frac{\mu_{it}^{P}(\theta)}{\mu_{it}^{P}(\theta^*)}\right)^{a_{ii}}\times \underbrace{\prod_{j \in \mathcal{N}_{i}} \left(\frac{\hat{\mu}_{jt}(\theta)}{\hat{\mu}_{jt}(\theta^*)}\right)^{a_{ij}}}_{\textit{KOO(Knowledge of Others)}} \label{eq:KOO} \\
        &\leq \left(\frac{\mu_{it}^{P}(\theta)}{\mu_{it}^{P}(\theta^*)}\right)^{a_{ii}} \times (t+1)^{1-a_{ii}},\label{eq:ratio_upperbound}
    \end{align}
    where the inequality follows by considering the worst case where the true state is never chosen by the neighbors while $\theta$ has been constantly chosen since the beginning. Now by invoking Lemma \ref{lemma:exponential_fast_convergence}, with probability at least $1-e^{-\gamma_{i\theta}t}$
    \begin{align*}
         \frac{\mu_{it}(\theta)}{\mu_{it}(\theta^*)} \leq e^{-a_{ii}\beta_{i\theta} t} \times (t+1)^{1-a_{ii}},
    \end{align*}
    % where last inequality holds for sufficiently small $0<\tilde{\beta}_{i\theta} < \beta_{i\theta}a_{ii}$.
\end{proof}
    % \begin{remark}
    %     Note that above corollary could hold for all $t \in [1, 2, \dots]$ if $\beta_{i\theta} > \frac{c}{a_{ii}}$ where $c >0$ is constant for which the inequality $(t+1)^{1-a_{ii}} \leq e^{ct}$ holds for all $t \in [1, 2, \dots]$ which may not hold if $\frac{c}{a_{ii}} \geq D_{1}(l_{i}(.|\theta^*)\|l_{i}(.|\theta))$, but this would not be problematic since our main concern is for the inequality to hold for large enough $t$.  
    % \end{remark}  
   
   However, what about the states that agents cannot distinguish from $\theta^*$? The answer lies in the knowledge of other users in the network who possess this capability which is charachterized by KOO term in Equation \eqref{eq:KOO}. In order to utilize this knowledge effectively, we must understand the properties of these users and how their expertise can benefit others. To accomplish this, we need to analyze the empirical distribution of the opinions declared by neighbors, taking into account the frequency of declaring each of $\theta$ and $\theta^*$ as their action, encapsulated in the parameters $n_{it}(\theta)$ and $n_{it}(\theta^*)$ which denote the number of times each of $\theta$ and $\theta^*$ are chosen by agent $i$ up to time $t$.

\subsection{The frequency of declaring true state}
In this section we investigate the $\hat\mu_{jt}(\theta^*)$ component of KOO term which is capturing the frequency of neighbors declaring true state $\theta^*$ as their opinion.

Lemma \ref{cor:1} was proved using the worst case lower bound $n_{it}(\theta^*)\geq 1$ (i.e. users don't take $\theta^*$ as their action up to time $t$), which is clearly an underestimation. To refine this, we must derive a non-trivial lower bound on the number of instances in which users select the true state $\theta^*$ as their action. 
% This matter will be tackled in this section. 
We first derive a lower bound on the belief of agents on $\theta^*$, and will subsequently use it to approximate the number of declared opinions matching the true state using some concentration inequalities. 
\begin{lemma} \label{lem:true_state_lowerbound}
    % For each agent $i \in [n]$ and time step $t \in [1, 2, \dots]$ we have:
    % % \begin{align*}
    % %     \mu_{it}^{S}(\theta^*|\omega_{i}^{t}) > \frac{1}{k(t+1)^{1-a_{ii}}},
    % % \end{align*}
    % \begin{equation*}
    %     \Probs\left(\mu_{it}(\theta^*) \leq \frac{1}{m(t+1)^{1-a_{ii}}}\right) \leq me^{-\tilde{\gamma}_{i}t},
    % \end{equation*}
    % where $\tilde{\gamma}_{i} := \gamma(\min_{\theta \neq \theta^*}\beta_{i\theta})$.
    For any agent $i \in [n]$, there exists $\gamma_i>0$ such that for all $t\in\mathbb{N}$ we have:
    \begin{equation*}
        \Probs\left(\mu_{it}(\theta^*|\omega_{i}^{t}) \leq \frac{1}{m(t+1)^{1-a_{ii}}}\right) \leq e^{-{\gamma}_{i}t}.
    \end{equation*}
\end{lemma}

\begin{proof}
    Reversing the inequality in Equation \eqref{eq:ratio_upperbound} we have:
    \begin{equation*}
        \frac{\mu_{it}(\theta^*)}{\mu_{it}(\theta)} \geq \left(\frac{\mu_{it}^{P}(\theta^*)}{\mu_{it}^{P}(\theta)}\right)^{a_{ii}} \times 1/(t+1)^{1-a_{ii}}.
    \end{equation*}
    For each non-$\theta^*$-identifiable $\theta$ in the above inequality, 
 we have $\left(\frac{\mu_{it}^{P}(\theta^*)}{\mu_{it}^{P}(\theta)}\right)^{a_{ii}} = 1$, so $\frac{\mu_{it}(\theta^*)}{\mu_{it}(\theta)} > 1/(t+1)^{1-a_{ii}}$. For the rest, by Lemma \ref{lemma:exponential_fast_convergence} we know that with probability at least $1-e^{-\gamma_{i\theta}t}$, the event $\frac{\mu_{it}^{P}(\theta^*)}{\mu_{it}^{P}(\theta)} > e^{a_{ii}\beta_{i\theta} t} > 1$ happens which implies $\frac{\mu_{it}(\theta^*)}{\mu_{it}(\theta)} > 1/(t+1)^{1-a_{ii}}$. By union bound we get 
 \begin{align*}
     &\Probs(\frac{\mu_{it}(\theta^*)}{\mu_{it}(\theta)} > 1/(t+1)^{1-a_{ii}}, \textit{For all} \; \theta \in \Theta) \\ &\geq 1-me^{-\gamma_{i}t}.
 \end{align*}
 Hence, we have \begin{equation*}
     \Probs(\mu_{it}(\theta^*|\omega_{i}^{t}) > \frac{1}{m(t+1)^{1-a_{ii}}}) \geq 1- me^{-\gamma_{i}t}.
 \end{equation*}
 where $\gamma_{i} := \gamma(\min_{\theta\neq\theta^*}\beta_{i\theta})$
\end{proof}
% \begin{corollary}
%     Using Borell-Contelli theorem we can see that the event $\mu_{it}^{S}(\theta^*|\omega_{i}^{t}) \leq \frac{1}{k(t+1)^{1-a_{ii}}}$ happens finitely many times.
% \end{corollary}
Since the agents are constructing the empirical distributions on a counting manner, we aim to derive at least how many times $\theta^*$ is chosen for large $t$. 
% \begin{proposition}\label{prop:zero_alpha}
%     \mahyar{ToBeFixed}
%      Define $\boldsymbol{\alpha}(0)$ to be an arbitrary initial vector of interest. We define an update rule for $\boldsymbol{\alpha}$ as $\boldsymbol{\alpha}(k+1) = A'\boldsymbol{\alpha}(k)$ for $k \in [0, 1, \dots]$ where $A'$ is an arbitraty square matrix with the norm of its rows smaller than 1 (i.e. $\|A'_{i,:}\|_{2} < 1$). Then we have $\|\boldsymbol{\alpha}(k)\|_{2} \rightarrow 0$ as $k \rightarrow \infty$.
% \end{proposition}
% \begin{proof}
%     By Perron–Frobenius Theorem for matrix $A'$ we have $\lambda_{\max}(A') <1$ which implies $\|\boldsymbol{\alpha}(k+1)\|_{2} = \|A'\boldsymbol{\alpha}(k)\|_{2} < \|\boldsymbol{\alpha}(k)\|_{2}$, which implies the result as $k \rightarrow \infty$.
% \end{proof}
From this point on, we will consider $t$ to be sufficiently large, and the inequalities that will be used would hold for large enough values of $t$. To further formalize this, we have the following Theorem.
\begin{lemma}\label{lem:lower_bound_star}
% For each agent $i \in [n]$, the number of times that action $\theta^*$ will be taken, up to time $t$, is $\Omega((t+1)^{1-\alpha_{i}(k)})$ with probability at least $1-m\cdot\sum_{j \in \mathcal{N}_{i}}e^{-c_{i}\delta^{2}(t+1)^{1-\alpha_{i}(k)}/2}$  where $\alpha_{i}(k)$ can be made arbitrarily small.\\
For each agent $i \in [n]$ there exists an $\alpha>0$, and $T_\alpha \in \mathbb{N}$ such that for all $t \geq T_\alpha$ we have:
\begin{align}
    n_{it}(\theta^*)>(t+1)^{1-\alpha}, 
\end{align}
with probability at least 
$1-e^{-(t+1)^{1-\alpha}}$.

\end{lemma}
% To analyze this, we first look at the times after $T$ in above corollary. 
\begin{proof}
% For each agent $i \in [n]$, 
Let $n_{it_{0}:t}(\theta^*
) := X_{it_{0}}^{\theta^*} + X_{it_{0}+1}^{\theta^*} + \dots + X_{it}^{\theta^*}$ be the number of times that $\theta^*$ is choosen by her, where $\{X_{\tau}\}_{\tau=t_{0}}^{t}$ are i.i.d. Bernouli random variables defined as $X_{i\tau}^{\theta^*} := \mathbf{1}_{c_{i\tau}}(\theta^*)$. By Lemma \ref{lem:true_state_lowerbound} and law of conditional expectation, for each time step $t$ we have: 
\begin{align*}
    &\Es[X_{t}]\\
    &= \mu_{it}(\theta^*| \mu_{it} > \frac{1}{m(t+1)^{1-a_{ii}}})\cdot\Probs(\mu_{it} > \frac{1}{m(t+1)^{1-a_{ii}}}) \\ 
    &+ \mu_{it}(\theta^*| \mu_{it} \leq \frac{1}{m(t+1)^{1-a_{ii}}})\cdot\Probs(\mu_{it} \leq \frac{1}{m(t+1)^{1-a_{ii}}}) \\
    &\geq \frac{1-m\cdot e^{-\gamma_{i}t}}{m(t+1)^{1-a_{ii}}}.
\end{align*}
So by using the Chernoff bound for $\delta \in (0, 1)$, we have 
\begin{align*}
     &\Probs\left(n_{it_{0}:t}(\theta^*) < (1-\delta)\cdot \sum_{\tau=t_{0}}^{t}\frac{1-m\cdot e^{-\gamma_{i}\tau}}{m(\tau+1)^{1-a_{ii}}}\right)\\
     &\leq \Probs\left(n_{it_{0}:t}(\theta^*) < (1-\delta)\cdot\Es[n_{it_{0}:t}(\theta^*)]\right) \\
    &\leq \exp(-\frac{\delta^2\cdot \Es[n_{it_{0}:t}(\theta^*)]}{2}),
\end{align*}
which implies with probability at least $1-e^{-\delta^2\cdot\Es[n_{it_{0}:t}(\theta^*)]/2}$ we have:
\begin{align*}
    &\frac{n_{it_{0}:t}(\theta^*)}{1-\delta} \geq \sum_{\tau=t_{0}}^{t}\frac{1-m\cdot e^{-\gamma_{i}\tau}}{m(\tau+1)^{1-a_{ii}}} \\
    &\geq \frac{(1-m\cdot e^{-\gamma_{i}t_{0}})}{m}\cdot\underbrace{\sum_{\tau=t_{0}}^{t}\frac{1}{(\tau+1)^{1-a_{ii}}}}_{:=h_{1-a_{ii}}(t_{0}, t)}\\
    &= \frac{(1-m\cdot e^{-\gamma_{i}t_{0}})}{m}\cdot h_{1-a_{ii}}(t_{0}, t)\\
    &\labelrel\geq{ineq:harmonic_sum} \frac{(1-m\cdot e^{-\gamma_{i}t_{0}})}{m}\int_{t_{0}+1}^{t+2}\frac{1}{\tau^{\alpha_{i}}}\mathbf{d}\tau \\
    &\geq \frac{(1-m\cdot e^{-\gamma_{i}t_{0}})}{m}\cdot\frac{(t+2)^{1-\alpha_{i}}-(t_{0}+1)^{1-\alpha_{i}}}{1-\alpha_{i}}  \\
    &\gtrsim c_{i}(t+1)^{1-\alpha_{i}} ,
\end{align*}
where $\alpha_{i} := 1-a_{ii}$, and $0<c_{i} < \frac{(1-m\cdot e^{-\gamma_{i}t_{0}})}{m(1-\alpha_{i})}$. \eqref{ineq:harmonic_sum} follows becuase of the inequliaty $\int_{1}^{N+1}\frac{1}{t^\alpha}\mathbf{d}t \leq \sum_{i={1}}^{N}\frac{1}{i^{\alpha}}$.
It could also be observed that, by choosing $t_{0} \in o(t)$, we have $h(t_{0}, t) \in \mathcal{O}(h(1, t))$, and $h(1, t_{0}) \in o(h(t_{0}, t))$, while we also need to take $t_{0} > \log(m)/\gamma_{i}$.

Hence we can write 
\begin{align}
    \frac{n_{it}(\theta^*)}{1-\delta} &\geq \frac{n_{it_{0}:t}(\theta^*)}{1-\delta} \\
    &\gtrsim c_{i}\cdot (t+1)^{1-\alpha_{i}}. \label{eq:true_parameter_lower_bound}
\end{align}
 By utilizing this, we can improve the inequality in Equation \eqref{eq:ratio_upperbound}; Since at the first place we bounded it at worst case considering that $n_{jt}(\theta^*) = 1$. Rewriting it we will have:
\begin{align}
    &\frac{\mu_{it}(\theta)}{\mu_{it}(\theta^*)} \lesssim \left(\frac{\mu_{it}^{P}(\theta)}{\mu_{it}^{P}(\theta^*)}\right)^{a_{ii}} \times \prod_{j \in \mathcal{N}_{i}}\left(\frac{t+1}{c_{j}\cdot(t+1)^{1-\alpha_{j}}}\right)^{a_{ij}} \nonumber\\
    &\leq \tilde{c}_{i}\left(\frac{\mu_{it}^{P}(\theta)}{\mu_{it}^{P}(\theta^*)}\right)^{a_{ii}}\times(t+1)^{\sum_{j \in \mathcal{N}_{i}}a_{ij}\alpha_{j}}, \label{eq:modified_upperbound}
    % &= \tilde{c}_{i}\left(\frac{\mu_{it}^{P}(\theta)}{\mu_{it}^{P}(\theta^*)}\right)^{a_{ii}}\times (t+1)^{\sum_{j \in \mathcal{N}_{i}}a_{ij}(1-a_{jj})} 
\end{align}
where $\tilde{c}_{i} = \prod_{j \in \mathcal{N}_{i}}(\frac{1}{c_{j}\cdot(1-\delta)})^{a_{ij}} > 1$. Note that since we are using \eqref{eq:true_parameter_lower_bound} for all the neighbors, by using Union Bound, \ref{eq:modified_upperbound} will hold with probability at least 
\begin{align*}
&1-\sum_{j \in \mathcal{N}_{i}}e^{-c_{i}\delta^{2}(t+1)^{1-\alpha_{i}}/2} 
\end{align*}
It appers the power of $(t+1)$ on the R.H.S of \ref{eq:modified_upperbound} is exhibiting an iterative pattern of $(t+1)^{\alpha_{i}(m)}$ for ($\alpha_{i}(m) : m\in[0, 1, \dots]$). where 
\[
\begin{cases} 
\alpha_{i}(0) = 1-a_{ii} \\
\alpha_{i}(k+1) = \sum_{j\neq i}a_{ij}\alpha_{j}(k) & :\forall k \in [0, 1, \dots]
\end{cases}
\]
writing this down for all users in matrix we get the following matrix form:
\[
\boldsymbol{\alpha}(m+1) = A'\boldsymbol{\alpha}(m),
\]
where $\boldsymbol{\alpha}(m) = [\alpha_{1}(m), \cdots, \alpha_{n}(m)]^{\top}$, and $A'$ is the adjacency matrix with its diagonals equal to zero.
The matrix forms now can be exploited to see 
\[\|\boldsymbol{\alpha}(m+1)\|_{\infty} < \|A'\|_{\infty}\cdot\|\boldsymbol{\alpha}(m)\|_{\infty}< \|A'\|_{\infty}^{m}\|\boldsymbol{\alpha}(0)\|_{\infty},\]
which implies that $\alpha_{i}$s could be made desirably small by increasing $m$ noticing that $\|A'\|_{\infty}<1$. This improves ~\eqref{eq:true_parameter_lower_bound} at the cost of decreasing $c_{i}$s in  
\begin{align*}
&1-\sum_{j \in \mathcal{N}_{i}}e^{-c_{i}\delta^{2}(t+1)^{1-\alpha_{i}}/2} - me^{-\gamma_{i}t} \\
&\gtrsim 1-e^{-(t+1)^{1-\alpha}}
\end{align*}
% but also keep in mind that every step of increasing $m$ in $\alpha(m)$, will add an additive term to the probability bound, meaning that we will have $n_{it}(\theta^*) \gtrsim c_{i}\cdot(t+1)^{1-\alpha_{i}(m)}$ with probability at least $1-m\cdot\sum_{j \in \mathcal{N}_{i}}e^{-c_{i}\delta^{2}(t+1)^{1-\alpha_{i}}/2}$.
\end{proof}

\subsection{The frequency of declaring a state $\theta\neq\theta^*$}
% \subsection{How many times $\theta$ will be chosen?}
\begin{proposition}\label{prop:finite_identifiable_choice}
    Each agent $i \in [n]$, chooses each of her $\theta^*$-identifiable states ($\theta$) finitely many times.
    % Meaning that there exists a time $T$, that for all $t > T$, $\{C_{it} = \theta\}$ never happens.
\end{proposition}
\begin{proof}
    We know that the event $\{c_{it}=\theta\}$ (which denotes if user $i$ announces $\theta$ at time $t$) occurs with probability $ p_{it}(\theta) := (\mu_{it}^{S}(\theta)) $. Using Corollary \ref{cor:1}, we can write
    \begin{align*}
        &\sum_{i=1}^{\infty}p_{it}(\theta)  \\ &=\sum_{t=1}^{\infty}\bigg[p_{it}(\theta | p_{it} > e^{-a_{ii}\beta_{i\theta}t})\cdot\Probs(p_{it}>e^{-a_{ii}\beta_{i\theta}t}) \\
        &+ \;p_{it}(\theta | p_{it} \leq e^{-a_{ii}\beta_{i\theta}t})\cdot\Probs(p_{it}<e^{-a_{ii}\beta_{i\theta}t})\bigg] \\
        &\leq \sum_{t=1}^{\infty} \left[e^{-\gamma_{i\theta}t} + e^{-a_{ii}\beta_{i\theta}t}\right]< \infty,
    \end{align*}
    where $c$ is sufficiently large so that for ($t>c$) Lemma \ref{cor:1} holds; Thus by using the Borell-Contelli Theorem, we deduce that the event $\{C_{it} = \theta\}$ happens finitely many times.
\end{proof}
\begin{remark}
    Note that above Proposition deosn't imply uniform boundedness of $n_{it}(\theta)$, meaning that there is a constant that it is smaller than for all times.
\end{remark}

Fix some $\theta\neq\theta^*$. As the next milestone, we aim to carefully construct upper bounds on $n_{it}(\theta)$ of the form
\begin{align}
\label{eq:ub}
n_{it}(\theta)\leq (1+t)^{\beta_i},
\end{align}
with probability at least $1-e^{-\gamma_i \sqrt{t}}$ 
for all $t\geq T_i$, for some $\beta_i,\gamma_i\geq0$ and $T_i\in\mathbb{N}$.
Observing that $n_{it}(\theta)\leq 1+t$ by definition, a trivial choice is $\beta_i=1,~\gamma_i=0,~T_i=1$. It turns out we can do much better. It proves convenient to define the notion of expert agents.
\begin{definition}
    % A set $\mathcal{J}_{\theta}$ is said to be a $\theta$-expert set, if it consists of agents in the network who can distinguish $\theta$ from $\theta^*$. The distance of an agent from this set
    % ---as typical as in graph litterature---is defined as the shortest distance between the agent and the agents inside the set.
    The set of $\theta$-expert agents $\mathcal{J}_{\theta}$ consists of agents in the network who can distinguish $\theta$ from $\theta^*$. The distance of an agent from this set
    is defined as the length of the shortest path connecting her to an agent in this set on the graph associated with A.\footnote{Note that the length of a path here is the number of edges on the path and not the sum of the weights of the edges on it.}     
    For any $i\in [n]$, we also define $\sigma_i$ to be the node immediately proceeding $i$ on the shortest path to this set (if there are multiple shortest paths, we choose one at random).
\end{definition}

Let us start by improving the choice of $(\beta_i,\gamma_i,T_i)$ for $\theta$-expert agents.

\begin{lemma}\label{lem:n_upper_bound}
For any $\theta$-expert agent $i \in [n]$ and any $\beta_i>\frac{3}{4}$, there exist $\gamma_{i} >0$ and $T_{i} \in \mathbb{N}$ such that 
\[
n_{it}(\theta) \leq (t+1)^{\beta_i}\]
with probability at least $1-e^{-\gamma_i \sqrt{t}}$ for $t\geq T_{i}$.
% For any $\theta$-expert agent $i \in [n]$, there exists $\gamma_{i} >0$, and $T_{i} \in \mathbb{N}$ such that 
% \[
% n_{it}(\theta) \leq (t+1)^{(1+\eta)/2}
% \]
% with probability at least $1-\exp(-2t^{\eta})$ for $t\geq T_{i}$.
    % For each agent $i \in [n]$ and her $\theta^*$-identifiable state $\theta$, 
    % we have $n_{it}(\theta) \lesssim c_{i}(t+1)^{(1+\eta)/2}$, with probability at least $1-\exp(-2t^{\eta})$.
\end{lemma}
\begin{proof}
    Using conditional expectation, and constructing a model, based on Bernoulli random variables to denote each time action $\theta$ is taken, we have $n_{it}(\theta) = X_{1} + \cdots + X_{t}$, using the same argument as in the proof of Proposition \ref{prop:finite_identifiable_choice}, we get $\Es[n_{it}(\theta)] \leq c +  \sum_{\tau=c+1}^{t} \left[e^{-\gamma_{i\theta}\tau} + e^{-a_{ii}\beta_{i\theta}\tau}\right] < \infty$, and by using the Hoeffding inequality \cite{Hoeffding} for bounded random variables we have we have
\begin{align*}
     &\Probs\left(n_{it}(\theta) > (\sum_{\tau=1}^{t}[e^{-\gamma_{i\theta}\tau}+e^{-a_{ii}\beta_{i\theta}\tau}]) +  t^{\beta_i}\right)\\
     &\leq \Probs\bigg(n_{it}(\theta) > \Es\left[n_{jt}(\theta)\right] + t^{\beta_i}\bigg) \\
    &\leq \exp(-2t^{2\beta_i-1}) \lesssim \exp(-\gamma_{i}\sqrt{t})
\end{align*}
for some choice of $\gamma_i$, which implies with probability $1-e^{-\gamma_i \sqrt{t}}$ we have $n_{it}(\theta) \lesssim (t+1)^{\beta_{i}}$
% \begin{align}
%     n_{it}(\theta) &\lesssim (t+1)^{\beta_{i}}
% \end{align}
% where $\beta_{i} > 1/2$.
\end{proof}

The following result enables us to come up with improved upper bounds of the form \eqref{eq:ub} for an agent exploiting the potentially improved bounds  of her neighbors.
\begin{lemma}
\label{lemma:recursion1}
Consider agent $i\notin\mathcal{J}_{\theta}$ and assume that her neighbors $j\in\mathcal{N}_i$ satisfy
\begin{align}
n_{jt}(\theta)\leq (1+t)^{\beta_j},
\end{align}
with probability at least $1-e^{-\gamma_j \sqrt{t}}$ 
for all $t\geq T_j$, for some $\{(\beta_j,\gamma_j,T_j)\}_{j\in\mathcal{N}_i}$.
Choose any $\beta_i>  a_{ii}+\sum_{j\neq i} a_{ij}\beta_j$, then there exists $\gamma_i>0$ and $T_i\in\mathbb{N}$ such that
\begin{align}
n_{it}(\theta)\leq (1+t)^{\beta_i},
\end{align}
with probability at least $1-e^{-\gamma_i \sqrt{t}}$ 
for all $t\geq T_i$.
\end{lemma}
\begin{proof}
We have:
\begin{align*}
    \frac{\mu_{it}(\theta)}{\mu_{it}(\theta^*)} &= \left(\frac{\mu_{it}^{P}(\theta)}{\mu_{it}^{P}(\theta^*)}\right)^{a_{ii}}\times \prod_{j \in \mathcal{N}_{i}} \left(\frac{\hat{\mu}_{jt}(\theta)}{\hat{\mu}_{jt}(\theta^*)}\right)^{a_{ij}} \\
    &\lesssim \tilde{c}_{i}  \prod_{j \in \mathcal{N}_{i}}  \left(\frac{(t+1)^{\beta_{j}}}{(t+1)^{1-\alpha_{j}}}\right)^{a_{ij}} \\
    &\leq\tilde{c}_{i}\times (t+1)^{(\sum_{j \in \mathcal{N}_{i}}a_{ij}(\beta_{j} - (1-\alpha_{i}^*)))}
    % &\leq \bar{c}_{i}\times (t+1)^{(\sum_{l\neq i}a_{il}\alpha_{i}^*(m) - (1-\frac{1+\eta}{2})a_{ij})}
\end{align*}
Which implies:
\begin{align*}
    {\mu_{it}(\theta)}
    \lesssim\tilde{c}_{i}\times (t+1)^{(\sum_{j \in \mathcal{N}_{i}}a_{ij}(\beta_{j} - (1-\alpha_{i}^*)))},
    % &\leq \bar{c}_{i}\times (t+1)^{(\sum_{l\neq i}a_{il}\alpha_{i}^*(m) - (1-\frac{1+\eta}{2})a_{ij})}
\end{align*}
with probability at least $p_t:=1- \sum_{j \in \mathcal{N}_{i}}(e^{-\gamma_{j}\sqrt{t}} + e^{-b_{j}\delta^{2}(t+1)^{1-\alpha_{j}}/2})$. So by law of conditional expectation we will have:
\begin{align*}
   &\Es\left[n_{it}(\theta)\right]\leq c\\
    &+\tilde{c}_{i}\times \sum_{\tau=c+1}^{t}(\tau+1)^{(\sum_{j \in \mathcal{N}_{i}}a_{ij}(\beta_{j} - (1-\alpha_{i}^*)))}\times p_\tau+(1-p_\tau) \\
    & \leq \tilde{c} + (t+1)^{(\sum_{j \in \mathcal{N}_{i}}a_{ij}(\beta_{j} - (1-\alpha_{i}^*))) + 1} \lesssim (t+1)^{\beta_{i}},
    % &\leq \bar{c}_{i}\times (t+1)^{(\sum_{l\neq i}a_{il}\alpha_{i}^*(m) - (1-\frac{1+\eta}{2})a_{ij})}
\end{align*}
where $\beta_{i}$ is chosen to satisfy $\beta_{i} > a_{ii} + \sum_{j \in \mathcal{N}_{i}}a_{ij}\beta_{j}$.
Using Hoeffding inequliaty we get:
\begin{align*}
     &\Probs\left(n_{it}(\theta) \gtrsim (t+1)^{\beta_{i}}\right) \leq \Probs\bigg(n_{it}(\theta) > c \\
     &+ \tilde{c}_{i}\sum_{\tau=c+1}^{t}(t+1)^{(\sum_{j \in \mathcal{N}_{i}}a_{ij}(\beta_{j} - (1-\alpha_{i}^*)))} +  t^{\beta_{i}}\bigg)\\
     &\leq \Probs\bigg(n_{it}(\theta) > \Es\left[n_{jt}(\theta)\right] + t^{\beta_{i}}\bigg) \\
    &\leq \exp(-2t^{2\beta_{i} - 1}) \lesssim \exp(-\gamma_{i}\sqrt{t})
\end{align*}
which proves the claim.
\end{proof}
Let us illustrate how one can use the above lemma to come up with non-trivial bounds of the form \eqref{eq:ub} for an agent $i$ with $\rm{dist}(\mathcal{J}_\theta,i)=1$. From definition, we have $\rm{dist}(\mathcal{J}_\theta,\sigma_i)=0$ ($\sigma_i$ is the neighbor of $i$ that is a $\theta$-expert). It thus follows from Lemma~\ref{lem:n_upper_bound} that for any choice of $\beta_{\sigma_i}>\frac{3}{4}$ there exists $\gamma_{\sigma_i}\geq0$ and $T_{\sigma_i}\in\mathbb{N}$ satisfying the bound of the form \eqref{eq:ub} for agent $\sigma_i$. Let us use the trivial triplet $(\beta,\gamma,T) = (1,0,1)$ for the rest of the neighbors of agent $i$. The condition on $\beta_i$ from the above lemma then becomes:
\[\beta_i>a_{ii}+\sum_{j\neq i, \sigma_i} a_{ij}+a_{i\sigma_i}\beta_{\sigma_i} = 1-a_{i\sigma_i}(1-\beta_{\sigma_i})\]
Recalling that any number greater than $\bar\beta^0:=\frac{3}{4}$ is a feasible choice for $\beta_{\sigma_i}$, the possible choices for $\beta_i$ becomes any $\beta_i>1-\frac{a_{i\sigma_i}}{4}$. 
Now, let 
\[\bar\beta^1:=\max_{i:\rm{dist}(\mathcal{J}_\theta,i)=1}1-\frac{a_{i\sigma_i}}{4}.\]
Then, for any agent $i$ at distance 1 from $\mathcal{J}_\theta$, any $\beta_i>\bar\beta^1$ is a feasible choice for a bound of the form \eqref{eq:ub}. Notice that this is a non-trivial bound since $\bar\beta^1<1$. 
Recursively applying the above argument, we can construct non-trivial bounds of the form \eqref{eq:ub} for agents at any distance from $\mathcal{J}_\theta$, as established in the next lemma.

\begin{lemma}\label{full_proof}
Let $h:=\max_{i \in[n]}\rm{dist}(\mathcal{J}_\theta, i)$. 
Consider the sequence $\{\bar\beta^l\}_{l=0}^{h}$ defined by the recursion 
\begin{align}
\bar{\beta}^{\ell+1} = \max_{i:\rm{dist}(\mathcal{J}_{\theta}, i)=\ell+1} 1-a_{i\sigma_i}(1-\bar{\beta}^{\ell}),
\end{align}
with $\bar\beta^0=\frac{3}{4}$. Then,\\
i) $\bar{\beta}^{\ell}<1$ for $l=0,1,\ldots,h$.\\
ii) For any $i\in[n]$ and $\beta_i>\bar\beta^{\rm{dist}(\mathcal{J}_{\theta}, i)}$, there exists $\gamma_i\geq0$ and $T_i\in\mathbb{N}$ such that
\begin{align}
n_{it}(\theta)\leq (1+t)^{\beta_i},
\end{align}
with probability at least $1-e^{-\gamma_i \sqrt{t}}$
for all $t\geq T_i$.
\end{lemma}
\begin{proof}
    Consider $i\in[n]$ with $\rm{dist}(\mathcal{J}_{\theta}, i)=\ell+1$.
    The proof is then the same as the argument made above the lemma, using the bound of the form in~\eqref{eq:ub} with $\beta_{\sigma_i} = \bar\beta^l$, and the trivial bounds (1,0,1) for the rest of the nodes.
\end{proof}
\begin{theorem}
    Within a strongly connected network of agents obeying the belief update rules governed by \eqref{eq:empiricals}-\eqref{eq:update_rule}, and assuming that for any $\theta \neq \theta^*$ there exists at least one $\theta$-expert agent, all the agents learn the true state $\theta^*$ with probability one in the sense that for all $i\in[n]$
\[\Probs\left(\lim\limits_{t \to \infty}\mu_{it}(\theta^*)=1\right) = 1\]
    % Within a strongly connected network of agents, having at least one $\theta$-expert agent for all $\theta \neq \theta^*$, following the update rule mentioned in Equation~\eqref{eq:update_rule} all the agents agent learns the true state $\theta^*$ along
    % the sample path $\omega \in \Omega$, almost surely as $t \rightarrow \infty$.
\end{theorem}
\begin{proof}
Equivalently, we may show that for any $\theta\neq\theta^*$:
\[\Probs\left(\lim\limits_{t \to \infty}\mu_{it}(\theta)=0\right) = 1\]
We consider two cases:
    % As $t \rightarrow \infty$, we will have 3 cases:
    \begin{enumerate}
        \item if agent $i$ can distinguish $\theta$ from $\theta^*$ the proof follows from Lemma~\ref{cor:1}.
        \item if $i\notin \mathcal{J}_\theta$, then the proof follows from the observation that
        for any $\tilde\beta_i>a_{ii}+\sum_{j\neq i}a_{ij}\bar\beta^{\rm{dist}(\mathcal{J}_\theta,j)}$, there exists $\gamma_i>0$ and $T_i\in\mathbb{N}$ such that with probability at least $1-e^{-\gamma_i \sqrt{t}}$ we have
        \begin{align}
        \mu_{it}(\theta)\leq (t+1)^{\tilde\beta_i-1},
        \end{align}
        for all $t\geq T_i$. The proof immediately follows noticing that  $a_{ii}+\sum_{j\neq i}a_{ij}\bar\beta^{\rm{dist}(\mathcal{J}_\theta,j)}<\sum_{j=1}^{n}a_{ij}=1$.       
    \end{enumerate} 
\end{proof}

\bibliographystyle{resources/IEEEbib}
\bibliography{IEEEabrv,root}
\end{document}